\documentclass[11pt]{article}
\pdfoutput=1

\usepackage{amsmath,amssymb}
\usepackage{bm}
\usepackage{natbib}
\usepackage[usenames]{color}
\usepackage{amsthm}

\usepackage{multirow} 
\usepackage{enumitem}

\usepackage[colorlinks,
linkcolor=red,
anchorcolor=blue,
citecolor=blue
]{hyperref}

\usepackage{setspace}
\usepackage[left=1in, right=1in, top=1in, bottom=1in]{geometry}

\usepackage{xcolor}

\usepackage{smile}

\title{\huge Generalization Error Bounds of Gradient Descent for Learning Over-parameterized Deep ReLU Networks}
\author
{
	Yuan Cao\thanks{Department of Computer Science, University of California, Los Angeles, CA 90095, USA; e-mail: {\tt yuancao@cs.ucla.edu}} 
	~~~and~~~
	Quanquan Gu\thanks{Department of Computer Science, University of California, Los Angeles, CA 90095, USA; e-mail: {\tt qgu@cs.ucla.edu}}
}
\date{}

\def\supp{\mathrm{supp}}
\def\rmI{\mathrm{I}}
\def\Tr{\mathrm{Tr}}
\def\poly{\mathrm{poly}}
\def\Diag{\mathrm{Diag}}
\def\bbvec{\mathrm{\mathbf{b}}}
\def\cb{\mathrm{\mathbf{c}}}

\def\cN{\mathcal{N}}

\newcommand{\la}{\langle}
\newcommand{\ra}{\rangle}
\def\CC{\textcolor{red}}

\begin{document}

\maketitle

\begin{abstract}
Empirical studies show that gradient-based methods can learn deep neural networks (DNNs) with very good generalization performance in the over-parameterization regime, where DNNs can easily fit a random labeling of the training data.
Very recently, a line of work explains in theory that with over-parameterization and proper random initialization, gradient-based methods can find the global minima of the training loss for DNNs. However, existing generalization error bounds are unable to explain the good  generalization performance of over-parameterized DNNs.
The major limitation of most existing generalization bounds is that they are based on uniform convergence and are independent of the training algorithm. In this work, we derive an algorithm-dependent generalization error bound for deep ReLU networks, and show that under certain assumptions on the data distribution, 
gradient descent (GD) with proper random initialization is able to train a sufficiently over-parameterized DNN to achieve arbitrarily small generalization error. Our work sheds light on explaining the good generalization performance of over-parameterized deep neural networks. 
\end{abstract}

\def\supp{\mathrm{supp}}
\def\rmI{\mathrm{I}}
\def\Tr{\mathrm{Tr}}
\def\poly{\mathrm{poly}}
\def\Diag{\mathrm{Diag}}
\def\bbvec{\mathrm{\mathbf{b}}}
\def\cb{\mathrm{\mathbf{c}}}

\def\cN{\mathcal{N}}

\def\CC{\textcolor{red}}

\section{Introduction}
\label{sec:intro}

Deep learning achieves great successes in almost all real-world applications ranging from image processing \citep{krizhevsky2012imagenet}, speech recognition \citep{hinton2012deep} to Go games \citep{silver2016mastering}. Understanding and explaining the success of deep learning has thus become a central problem for theorists. One of the mysteries is that the neural networks used in practice are often heavily over-parameterized 
such that they can even fit random labels to the input data \citep{zhang2016understanding}, while they can still achieve very small generalization error (i.e., test error) when trained with real labels. 

There are multiple recent attempts towards answering the above question and demystifying the success of deep learning. 
\citet{soudry2016no,safran2016quality,arora2018optimization,haeffele2015global,nguyen2017loss} showed that over-parameterization can lead to better optimization landscape.
\cite{li2018learning,du2018gradient}
proved that with proper random initialization, gradient descent (GD) and/or stochastic gradient descent (SGD) provably find the global minimum for training over-parameterized one-hidden-layer ReLU networks. \citet{arora2018convergence} analyzed the convergence of GD to global optimum for training a deep linear neural network under a set of assumptions on the network width and initialization. \citet{du2018gradientdeep,allen2018convergence,zou2019gradient} studied the convergence of gradient-based method for training over-parameterized deep nonlinear neural networks. Specifically, \citet{du2018gradientdeep} proved that gradient descent can converge to the global minima for over-parameterized deep neural networks with smooth activation functions. 
\citet{allen2018convergence,zou2019gradient} independently proved the global convergence results of GD/SGD for deep neural networks with ReLU activation functions in the over-parameterization regime. However, in such an over-parametrized regime, the training loss function of deep neural networks may have potentially infinitely many global minima,  but not all of them can generalize well. Hence, convergence to the global minimum of the training loss is not sufficient to explain the good generalization performance of GD/SGD. 

There are only a few studies on the generalization theory for learning neural networks in the over-parameterization regime.
\citet{brutzkus2017sgd} showed that SGD learns over-parameterized networks that provably generalize on linearly separable data. 
\citet{song2018mean} showed that when training two-layer networks in a suitable scaling limit, the SGD dynamic is captured by a certain non-linear partial differential equation with nearly ideal generalization error. 
\cite{li2018learning} relaxed the linear separable data assumption and proved that SGD learns an over-parameterized network with a small generalization error when the data comes from mixtures of well-separated distributions. 
\citet{allen2018learning} proved that under over-parameterization, SGD or its variants can learn some notable hypothesis classes, including two and three-layer neural networks with fewer parameters. \citet{arora2019fine} provided a generalization bound of GD for two-layer ReLU networks based on a fine-grained analysis on how much the network parameters can move during GD. Nevertheless, all these results are limited to two or three layer neural networks, and cannot explain the good generalization performance of gradient-based methods for \emph{deep} neural networks. For deep neural networks, existing generalization error bounds \citep{neyshabur2015norm,bartlett2017spectrally,neyshabur2017pac,golowich2017size,dziugaite2017computing,arora2018stronger,li2018tighter,neyshabur2018towards,wei2018regularization} are mostly based on uniform convergence and independent of the training algorithms. \citet{daniely2017sgd} established a generalization bound for over-parameterized neural networks trained with one-pass SGD. However, they considered a setting where the training of hidden layers are neglectable and only the output layer training is effective.

In this paper, we aim to answer the following question:
\begin{center}
\emph{Why gradient descent can learn an over-parameterized deep neural network that generalizes well?}
\end{center}
Specifically, we consider learning deep fully connected ReLU networks with cross-entropy loss using over-parameterization and gradient descent. 

\subsection{Our Main Results and Contributions}

The following theorem gives an informal version of our main results. 

\begin{theorem}[Informal version of Corollaries~\ref{col:randomfeaturegeneralization},\ref{col:conjugatekernelgeneralization}]\label{thm:convergence_gd_informal}
Under certain data distribution assumptions, for any $\epsilon > 0$, if the number of nodes per each hidden layer is set to $\tilde \Omega( \epsilon^{-14})$ and the sample size $n = \tilde\Omega(\epsilon^{-4})$, then with high probability, gradient descent with properly chosen step size and random initialization method learns a deep ReLU network and achieves a population classification error at most $\epsilon$.
\end{theorem}
Here in 
Theorem~\ref{thm:convergence_gd_informal} we use $\tilde O(\cdot)$ and $\tilde \Omega(\cdot)$ to hide some logarithmic terms in standard Big-O and Big-Omega notations. The result of Theorem~\ref{thm:convergence_gd_informal} holds for ReLU networks with arbitrary constant number of layers, as long as the data distribution satisfies certain separation condition, which will be discussed in Section~\ref{section:datadistributionassumptions}. 

\noindent\textbf{Our contributions.} 
Our main contributions are as follows:
\begin{itemize}
	\item We provide a generalization error bound specifically suitable for wide neural networks of arbitrary depth. The bound enjoys better dependency in terms of the network width compared with	existing generalization error bounds for deep neural networks \citep{neyshabur2015norm,bartlett2017spectrally,neyshabur2017pac,golowich2017size,arora2018stronger,li2018tighter,wei2018regularization}. Moreover, we also provide an optimization result on the convergence of gradient descent for over-parameterized neural networks. Combining these two results together gives an algorithm dependent bound of expected error that is independent of the network width.
    \item We investigate two types of data distribution assumptions, and show that under each of them, gradient descent can train an over-parameterized neural network to achieve $\epsilon$ expected error provided $\tilde O(\epsilon^{-4})$ training examples. The data distribution assumptions we consider in this paper are standard and have been studied in recent literature. This demonstrates that our analysis can give meaningful generalization bounds even for very wide neural networks, and can provide insights on the practical success of over-parameterized neural networks.
\end{itemize}

\subsection{Notation}
Throughout this paper, scalars, vectors and matrices are denoted by lower case, lower case bold face, and upper case bold face letters respectively. For a positive integer $n$, we denote $[n] = \{1,\dots,n\}$. For a vector $\xb = (x_1,\dots,x_d)^\top$, we denote by $\|\xb\|_p=\big(\sum_{i=1}^d |x_i|^p\big)^{1/p}$, $\|\xb\|_\infty = \max_{i=1,\dots,d} |x_i|$, and $\|\xb\|_0 = |\{x_i:x_i\neq 0,i=1,\dots,d\}|$ the $\ell_p$, $\ell_\infty$ and $\ell_0$ norms of $\xb$ respectively. We use $\text{Diag}(\xb)$ to denote a square diagonal matrix with the entries of $\xb$ on the main diagonal. For a matrix $\Ab = (A_{ij})\in \RR^{m\times n}$, we use $\|\Ab\|_2$ and $\|\Ab\|_F$ to denote the spectral norm (maximum singular value) and Frobenius norm of $\Ab$ respectively. We also denote by $\|\Ab\|_0$ the number of nonzero entries of $\Ab$. We denote by $S^{d-1} = \{ \xb\in\RR^d:\| \xb\|_2 =1\}$ the unit sphere in $\RR^d$. For a function $f:\RR^d \rightarrow \RR$, we denote by $\| f(\cdot) \|_\infty = \inf\{ C \geq 0: |f(\xb)| \leq C \text{ for almost every }\xb \}$ the essential supreme of $f$.

We use the following standard asymptotic notations. 
For two sequences $\{a_n\}$ and $\{b_n\}$, we write $a_n = O(b_n)$ if $a_n\le C_1 b_n$ for some absolute constant $C_1> 0$, and $a_n = \Omega (b_n)$ if $a_n\ge C_2 b_n$ for some absolute constant $C_2>0$. In addition, we use $\tilde O(\cdot)$ and $\tilde \Omega(\cdot)$ to hide some logarithmic terms in Big-O and Big-Omega notations.


\section{Additional Related Work}\label{sec:related work}

There is a huge body of literature towards building the foundations of deep learning, and we are not able to include every work in this paper. In this section, we briefly review and comment additional work that is most related to ours and was not discussed in Section \ref{sec:intro}.

\noindent\textbf{Representation power of deep neural networks.} A line of research has shown that deeper neural networks have higher expressive power
\citep{telgarsky2015representation, telgarsky2016benefits, lu2017expressive, liang2016deep, yarotsky2017error, yarotsky2018optimal, hanin2017universal, hanin2017approximating} than shallow neural networks. This to certain extent explains the advantage of deep neural networks with over-parameterization. \citet{lin2018resnet} proved that ResNet \citep{he2016deep} with one hidden node per layer is a universal approximator to any Lebesgue integrable function.

\noindent\textbf{Optimization landscape of neural networks.} Many studies \citep{haeffele2015global,kawaguchi2016deep,freeman2016topology,hardt2016identity,safran2017spurious,xie2017diverse,nguyen2017loss,soltanolkotabi2017theoretical,zhou2017critical,yun2017global,du2018power,venturi2018neural,gao2019learning} investigated the optimization landscape of neural networks with different activation functions. However, these results only apply to one-hidden layer neural networks, or deep linear networks, or rely on some stringent assumptions on the data and/or activation functions. In fact, they do not hold for non-linear shallow neural networks \citep{yun2019small} or three-layer linear neural networks \citep{kawaguchi2016deep}. Furthmore, \citet{yun2019small} showed that small nonlinearities in activation functions create bad local minima in neural networks.

\noindent\textbf{Implicit bias/regularization of GD and its variants.}
A bunch of papers  \citep{gunasekar2017implicit,soudry2017implicit,ji2019implicit,gunasekar2018characterizing,gunasekar2018implicit,nacson2018stochastic,li2018algorithmic} have studied implicit regularization/bias of GD, stochastic gradient descent (SGD) or mirror descent for matrix factorization, logistic regression, and deep linear networks. However, generalizing these results to deep non-linear neural networks turns out to be challenging and is still an open problem.
 
 \noindent\textbf{Connections between deep learning and kernel methods.}
 \cite{daniely2017sgd} uncovered the connection between deep neural networks with kernel methods and showed that SGD can learn a function that is comparable with the best function in the conjugate kernel space of the network. \citet{jacot2018neural} showed that the evolution of a DNN during training can be described by a so-called neural tangent kernel, which makes it possible to study the training of DNNs in the functional space.
 \citet{belkin2018understand,liang2018just} showed that good generalization performance of overfitted/interpolated classifiers is not only an intriguing feature for deep learning, but also for kernel methods.

\noindent\textbf{Recovery guarantees for shallow neural networks.} A series of work \citep{tian2017analytical,brutzkus2017globally, li2017convergence,soltanolkotabi2017learning,du2017convolutional, du2017gradient,zhong2017recovery,zhang2018learning,cao2019tight} have attempted to study shallow one-hidden-layer neural networks with ground truth parameters, and proved recovery guarantees for gradient-based methods such as gradient descent (GD) and stochastic gradient descent (SGD). However, the assumption of the existence of ground truth parameters is not realistic and the analysis of the recovery guarantee can hardly be extended to deep neural networks. Moreover, many of these studies need strong assumptions on the input distribution such as Gaussian, sub-Gaussian or symmetric distributions.

\noindent\textbf{Distributional view of over-parameterized networks.} 
\citet{mei2018mean,chizat2018global,sirignano2019mean,rotskoff2018neural,wei2018regularization} took a distributional view of over-parametrized networks, used mean field analysis to show that the empirical distribution of the two-layer neural network parameters can be described as a Wasserstein gradient flow, and proved that Wasserstein gradient flow converges to global optimima under certain structural assumptions. However, their results are limited to two-layer infinitely wide neural networks. Very recently, \citet{yang2019scaling} studied the scaling limit of wide multi-layer neural networks.

\section{Problem Setup and Training Algorithm}\label{sec:preliminaries}
In this paper, for the sake of simplicity, we study the binary classification problem on some unknown but fixed data distribution $\cD$ over $\RR^d \times \{+1,-1\}$. An example $(\xb,y)$ drawn from $\cD$ consists of the input $\xb \in \RR^d$ and output label $y\in \{+1,-1\}$. We denote by $\cD_\xb$ the marginal distribution of $\xb$. Given an input $\xb$, we consider predicting its corresponding label $y$ using a deep neural network with the ReLU activation function $\sigma(z):= \max\{0,z\}$. We consider $L$-hidden-layer neural networks with $m_l$ hidden nodes on the $l$-th layer for $l=1,\ldots,L$. The neural network function (mapping) is defined as follows
\begin{align*}
    f_{\Wb}(\xb) = \vb^\top \sigma ( \Wb_{L}^{\top} \sigma ( \Wb_{L-1}^{\top} \cdots \sigma( \Wb_{1}^{\top} \xb )\cdots)),
\end{align*}
where $\sigma(\cdot)$ denotes the entry-wise ReLU activation function (with a slight abuse of notation), $\Wb_{l} = (\wb_{l,1},\ldots,\wb_{l,m_l}) \in \RR^{m_{l-1}\times m_{l}}$, $l=1,\ldots,L$ are the weight matrices, and $\vb\in (\mathbf{1}^\top, -\mathbf{1}^\top)^\top \in\{-1,+1\}^{m_L}$ is the fixed output layer weight vector with half $1$ and half $-1$ entries. In particular, set $m_0 = d$. We denote by $\Wb=\{\Wb_l\}_{l=1}^L$ the collection of matrices $\Wb_1,\dots,\Wb_L$.

Given $n$ training examples $(\bx_1,y_1),\ldots,(\bx_n,y_n)$ drawn independently from $\cD$, the training of the neural network can be formulated as an empirical risk minimization (ERM) problem as follows:
\begin{align}\label{eq:problemdefinition}
    \min_{\Wb} L_S(\Wb) = \frac{1}{n} \sum_{i=1}^n \ell[y_i \cdot f_{\Wb}(\bx_i) ], 
\end{align}
where $S = \{ (\bx_1,y_1),\ldots,(\bx_n,y_n) \}$ is the training sample set, and $\ell(z)$ is the loss function. In this paper, we focus on cross-entropy loss function, which is in the form of $\ell(z) = \log[1+\exp(-z)]$. Our result can be extended to other loss functions such as square loss and hinge loss as well.


\subsection{Gradient Descent with Gaussian Initialization}
Here we introduce the details of the algorithm we use to solve the empirical risk minimization problem \eqref{eq:problemdefinition}.  The entire training algorithm is summarized in Algorithm~\ref{alg:gdrandominit}. 

\begin{algorithm}[h]
\caption{Gradient descent for DNNs starting at Gaussian initialization}
\begin{algorithmic}
\REQUIRE Training data $\{(\xb_i,y_i)\}_{i=1}^n$, number of iterations $K$, step size $\eta$.
\STATE Generate each entries of $\Wb_l^{(0)}$ independently from $N(0,2/m_l)$, $l\in[L]$. 
\FOR{$k=0,1,2,\ldots, K-1$}
\STATE $\Wb_{l}^{(k)} = \Wb_{l}^{(k-1)} - \eta \nabla_{\Wb_l} L_{S}(\Wb_{l}^{(k-1)}),~l\in[L]$.
\ENDFOR
\STATE $k^* = \argmin_{k\in \{ 0,\ldots,K-1 \}} -\frac{1}{n}\sum_{i=1}^n \ell'\big( y_i \cdot f_{\Wb}^{(k)}(\bx_i) \big) $.
\ENSURE $\Wb^{(0)},\ldots,\Wb^{(K)}$
\end{algorithmic}\label{alg:gdrandominit}
\end{algorithm}

In detail, Algorithm~\ref{alg:gdrandominit} consists of two stages: random initialization and gradient descent (GD). In the random initialization stage, we initialize $\Wb^{(0)} = \{ \Wb_{l}^{(0)} \}_{l = 1}^L$ via Gaussian initialization  for all $l\in [L]$, where each entries of $\Wb_l^{(0)}$ are generated independently from $N(0,2/m_l)$. Note that the initialization scheme of $\Wb^{(0)}$ is essentially the initialization proposed in \citet{he2015delving}.
In the gradient descent stage, 
we do gradient descent starting from $\Wb^{(0)} $, where $\eta>0$ is the step size, and the superscript $(k)$ is the iteration index of GD. One can also use stochastic gradient descent (SGD) to solve \eqref{eq:problemdefinition}, and our theory can be extended to SGD as well. Due to space limit, we only consider GD in this paper.

\section{Main Theory}\label{section:maintheory}
In this section we present our main result. We first introduce several  assumptions.

\begin{assumption}\label{assump:normalizeddata}
The input data are normalized: $\supp(\cD_x) \subseteq S^{d-1}$.
\end{assumption}

Assumption~\ref{assump:normalizeddata} is widely made in most existing work on over-parameterized neural networks \citep{li2018learning,allen2018convergence,du2017gradient,du2018gradient,zou2019gradient}. This assumption can be relaxed to the case that $ c_1 \leq \| \xb \|_2 \leq c_2$ for all $\xb \in \supp(\cD_x)$, where $c_2 > c_1 > 0$ are absolute constants. Such relaxation will not affect our final generalization results.

\begin{assumption}\label{assump:m_scaling} We have $M/m = O(1)$,
where $M = \max\{ m_1,\ldots, m_L\}$, $m = \min\{ m_1,\ldots, m_L\}$.
\end{assumption}

Assumption~\ref{assump:m_scaling} essentially says that the width of each layer in the deep neural network is in the same order, and the neural work architecture is balanced. Throughout this paper, we always assume Assumptions~\ref{assump:normalizeddata} and \ref{assump:m_scaling} hold. We therefore omit them in our theorem statements. 


For the ease of exposition we introduce the following definitions.
\begin{definition}
For the collection of random parameters $\Wb^{(0)} = \{\Wb_l^{(0)}\}_{l=1}^L$ generated in Algorithm \ref{alg:gdrandominit}, we call 
$$\cW_{\tau} := \big\{ \Wb = \{ \Wb_l\}_{l=1}^L: \| \Wb_l - \Wb_l^{(0)}\|_F \leq \tau,~ l\in[L] \big\}
$$
the $\tau$-neighborhood of $\Wb^{(0)}$. 
\end{definition}
The definition of $\cW_{\tau}$ is motivated by the observation that in a small neighborhood of initialization, deep ReLU networks satisfy good scaling and landscape properties. It also provides a small subset of the entire hypothesis space and enables a sharper capacity bound based on Rademacher complexity for the generalization gap between empirical and generalization errors. 

\begin{definition}
For a collection of parameter matrices $\Wb = \{\Wb_l\}_{l=1}^L$, we define its empirical surrogate error $\cE_S(\Wb)$ and population surrogate error $\cE_\cD(\Wb)$ as follows:
\begin{align*}
    \cE_S(\Wb) := -\frac{1}{n} \sum_{i=1}^n \ell'\big[y_i\cdot f_{\Wb}(\bx_i)\big],~
    \cE_\cD(\Wb) := \EE_{(\xb,y) \sim \cD} \big\{ - \ell'\big[y\cdot   f_{\Wb}(\xb)\big] \big\}.
\end{align*}
\end{definition}
The intuition behind the definition of surrogate error is that, for cross-entropy loss we have $-\ell'(z) = 1/[1 + \exp(z)]$, which can be seen as a smooth version of the indicator function $\ind\{z < 0\}$, and therefore $-\ell'[y\cdot f_{\Wb}(\xb)]$ is related to the classification error of the neural network. Surrogate error plays a pivotal role in our generalization analysis: on the one hand, it is closely related to the derivative of the empirical loss function. On the other hand, by $-2\ell'(z) \geq \ind\{z < 0\}$, it also provides an upper bound on the classification error. It is worth noting that the surrogate error is comparable with the ramp loss studied in margin-based generalization error bounds \citep{neyshabur2015norm,bartlett2017spectrally,neyshabur2017pac,golowich2017size,arora2018stronger,li2018tighter} in the sense that it is Lipschitz continuous in $\Wb$, which ensures that $\cE_S(\Wb)$ concentrates on $\cE_\cD(\Wb)$ uniformly over the parameter space $\cW_{\tau}$.

\subsection{Generalization and Optimization of Over-parameterized Neural Networks}
In this section, we provide (i) a generalization bound for neural networks with parameters in a neighborhood of random initialization, (ii) a convergence guarantee of gradient descent for training over-parameteried neural networks. Combining these two results gives a bound on the expected error of neural networks trained by gradient descent.


\begin{theorem}\label{thm:empiricalell'2populationerror}
For any $\delta>0$, there exist absolute constants $\overline{C}, \overline{C}',\underline{C}$ such that, if 
\begin{align*}
    m\geq \overline{C}\max\{ L^2\log(mn/\delta), L^{-8/3}\tau^{-4/3}\log[m/(\tau\delta)] \}, ~\tau \leq \underline{C} L^{-6} [\log(m)]^{-3/2},
\end{align*}
then with probability at least $1 - \delta$, 
\begin{align*}
    \PP_{(\xb,y)\sim \cD} \big[y \cdot f_{\Wb}(\xb) < 0 \big] \leq 2\cdot \cE_S(\Wb)  + \overline{C}' \big[  L\tau\cdot \sqrt{m/n} + L^4\sqrt{m\log(m)} \tau^{4/3}\big]
\end{align*}
for all $\Wb \in \cW_{\tau}$. 
\end{theorem}

\begin{remark}
 For neural networks initialized with He initialization \citep{he2015delving}, the generalization bound given by Theorem~\ref{thm:empiricalell'2populationerror} has a better dependency in network width $m$ compared with existing uniform convergence based generalization error bounds \citep{neyshabur2015norm,bartlett2017spectrally,neyshabur2017pac,golowich2017size,arora2018stronger,li2018tighter,wei2018regularization}. For instance, $\Wb\in \cW_\tau$ implies $\|\Wb_l^\top-\Wb_l^{(0)\top}\|_{2,1} \leq \sqrt{m}\tau$ and $\| \Wb_l \|_2 = \tilde O(1)$. Plugging these bounds into the generalization bound given by \citet{bartlett2017spectrally}
\begin{align*}
    \tilde O\Bigg(  \frac{\| \vb \|_2 }{\sqrt{n}} \prod_{l=1}^L \| \Wb_l \|_2 \Bigg[ \sum_{l=1}^L \frac{\| \Wb_l^{\top} - \Wb_l^{(0)\top} \|_{2,1}^{2/3}}{ \| \Wb_l \|_{2}^{2/3} } \Bigg]^{3/2} \Bigg)
\end{align*}
or the bound given by \citet{neyshabur2017pac}
\begin{align*}
    \tilde O\Bigg( \frac{L \| \vb \|_2 }{\sqrt{n}} \prod_{l=1}^L \| \Wb_l \|_2 \Bigg[ \sum_{l=1}^L \frac{ (\sqrt{m} \| \Wb_l - \Wb_l^{(0)} \|_{F})^{2}}{ \| \Wb_l \|_{2}^{2} } \Bigg]^{1/2} \Bigg)
\end{align*}
results in a generalization bound of the order $\tilde O( m\tau /\sqrt{n} )$. In comparison, when $\tau$ is small enough, our bound on the generalization gap is in the order of $\tilde O(\tau\cdot \sqrt{m/n})$, which has a better dependency in $m$. Note that for over-parameterized neural networks, gradient descent indeed converges to a global minima that is very close to initialization, as we will show in Theorem~\ref{thm:convergence_gd_empirical}. Therefore, while the previously mentioned uniform convergence based generalization bounds hold for more general settings and are more suitable when the weight matrices are not close enough to random initialization, our bound in Theorem~\ref{thm:empiricalell'2populationerror} provides a sharper result that is specifically designed for the over-parameterized setting. 
\end{remark}

Theorem~\ref{thm:empiricalell'2populationerror} in particular suggests that if gradient descent finds a parameter configuration with small surrogate error in $\cW_{R m^{-1/2}}$ for some $R$ independent of $m$, then the obtained neural network has a generalization bound \textit{decreasing} in $m$. The following lemma shows that under a gradient lower bound assumption, gradient descent indeed converges to a global minima in $\cW_{R m^{-1/2}}$ with $R$ independent of $m$.

\begin{theorem}\label{thm:convergence_gd_empirical}
Suppose that the training loss function $L_S(\Wb)$ satisfies the following inequality
\begin{align}\label{eq:gradientlowerbound_assump}
    \big\| \nabla_{\Wb_{L}} L_S( \Wb) \big\|_F \geq B \sqrt{m} \cdot \cE_S(\Wb)
\end{align}
for all $\Wb\in \cW_\tau$, where $B$ is independent of $m$, and $\tau = \tilde O( B^{-1}\epsilon^{-1} m^{-1/2} )$. For any $\epsilon,\delta>0$, there exist absolute constants $\overline{C},\underline{C}$ and $m^* = \tilde O( L^{12}B^{-4}\epsilon^{-2} ) \cdot \log(1/\delta)$ such that, if $m\geq m^*$, then with probability at least $1 - \delta$, Algorithm~\ref{alg:gdrandominit} with step size $\eta = O(L^{-3}B^2 m^{-1})$ generates $K = \tilde O(L^3B^{-4} \epsilon^{-2})$ iterates $\Wb^{(1)},\ldots,\Wb^{(K)}$ that satisfy:
\begin{enumerate}[label=(\roman*)]
    \item $\Wb^{(k)} \in \cW_{\tau}$, $k\in [K]$.
    \item There exists $k \in \{ 0,\ldots,K-1 \}$ such that
    $\cE_S(\Wb^{(k)}) \leq \epsilon$. 
\end{enumerate}
\end{theorem}
\begin{remark}\label{remark:gradientlowerbound}
The gradient lower bound assumption \eqref{eq:gradientlowerbound_assump} is by no means an unrealistic assumption. In fact, this assumption has been verified by several papers \citet{allen2018convergence,zou2019gradient,zou2019improved} under the assumption that $\| \xb_i - \xb_j \|_2\geq \phi$ for all $i,j\in[n]$, where $\phi>0$ is an absolute constant. The corresponding value of $B$ under this assumption is $ \Omega(\poly(\phi,n^{-1}))$. 
\end{remark}

Combining Theorems~\ref{thm:empiricalell'2populationerror} and~\ref{thm:convergence_gd_empirical} directly gives the following corollary:
\begin{corollary}\label{col:populationerrorwithgradientlowerbound}
Suppose that the training loss function $L_S(\Wb)$ satisfies inequality \eqref{eq:gradientlowerbound_assump} for all $\Wb\in \cW_\tau$, where $B$ is independent of $m$, and $\tau = \tilde O( B^{-1}\epsilon^{-1} m^{-1/2} )$. For any $\epsilon,\delta>0$, there exist absolute constants $\overline{C},\underline{C}$ and $m^* = \tilde O( L^{12}B^{-4}\epsilon^{-2} ) \cdot \log(1/\delta)$ such that, if $m\geq m^*$, then with probability at least $1 - \delta$, Algorithm~\ref{alg:gdrandominit} with step size $\eta = O(L^{-3}B^2 m^{-1})$ finds a point $\Wb^{(k)}$ that satisfies
\begin{align*}
    &\PP_{(\xb,y)\sim \cD} \big[y \cdot f_{\Wb}(\xb) < 0 \big] \leq \epsilon + \tilde O\big( L^2B^{-1}\epsilon^{-1}\cdot n^{-1/2} + L^4 B^{-4/3}\epsilon^{-4/3}m^{-1/6}\big)
\end{align*}
within $K = \tilde O(L^3B^{-4} \epsilon^{-2})$ iterations.
\end{corollary}

As we discussed in Remark~\ref{remark:gradientlowerbound}, if the pairwise distance between training inputs can be lower bounded by a constant $\phi$, then \eqref{eq:gradientlowerbound_assump} holds with $B = O(\poly(\phi,n^{-1}))$. However, plugging this value of $B$ into the population error bound in Corollary~\ref{col:populationerrorwithgradientlowerbound} will give a bound $\tilde O(\poly(n)\cdot n^{-1/2})$ (when $m$ is large enough) which is vacuous and does not decrease in sample size $n$. We remark that this result is natural, because $B = \Omega(\poly(\phi,n^{-1}))$ corresponds to the condition that data inputs are separated, and in fact no assumption on the distribution of labels is made through out our analysis. Suppose that the labels are simply Rademacher variables and are independent of inputs, then clearly the expected error of any classifier cannot go below $1/2$, no matter how many training samples are used to learn the classifier. In the next subsection, we study particular data distribution assumptions under which \eqref{eq:gradientlowerbound_assump} holds with a $B$ that is not only independent of $m$, but also independent of $n$.


\subsection{Generalization Error Bounds under Specific Data Distribution Assumptions}\label{section:datadistributionassumptions}
In this section we introduce two specific data distributions that have been studied in the literature, and show that if one of them holds, then \eqref{eq:gradientlowerbound_assump} holds with a $B$ independent of $m$ and $n$. Assumption~\ref{assump:nonlinearseparable} below is related to a random feature model studied in \citet{rahimi2009weighted}.

\begin{assumption}[Separable by Random ReLU Feature]\label{assump:nonlinearseparable}
Denote by $p(\overline\ub)$ the density of standard Gaussian random vectors. Define
\begin{align*}
\cF = \bigg\{ f(\xb_i) = \int_{\RR^d} c(\overline\ub)\sigma(\overline\ub^\top\xb_i) p(\overline\ub) \mathrm{d}\overline\ub : \| c(\cdot) \|_{\infty} \leq 1 \bigg\}.
\end{align*}
We assume that there exist an $f(\cdot)\in \cF$ and a constant $\gamma > 0$ such that
$ y_i\cdot f(\xb_i) \geq \gamma $ for all $i\in[n]$.
\end{assumption}

$\cF$ defined in Assumption~\ref{assump:nonlinearseparable} corresponds to the random feature function class studied in \citet{rahimi2009weighted} when the feature function is chosen to be ReLU. Assumption~\ref{assump:nonlinearseparable} essentially states that there exists a function $f$ in the function class $\cF$ that can separate the data distribution $\cD$ with a constant margin $\gamma$. According to the definition of $\cF$, each value of $\overline{\ub}$ can be considered as a node in an infinite-width one-hidden-layer ReLU network, and the corresponding product $c(\overline{\ub}) p(\overline{\ub})$ can be considered as the second-layer weight. Therefore $\cF$ contains infinite-width one-hidden-layer ReLU networks whose second-layer weights decay faster than $p(\overline\ub)$. Also note that Assumption~\ref{assump:nonlinearseparable} is strictly milder than linearly separable assumption.

The following corollary gives an expected error bound of neural networks trained by gradient descent under Assumption~\ref{assump:nonlinearseparable}. 


\begin{corollary}\label{col:randomfeaturegeneralization}
Under Assumption~\ref{assump:nonlinearseparable}, for any $\epsilon ,\delta > 0$, there exist 
\begin{align*}
    &m^*(\epsilon, L, \gamma,\delta) = \tilde O(\poly(2^L, \gamma^{-1}))\cdot \epsilon^{-14}\cdot \log(1/\delta),\\
    &n^*(\epsilon, L, \gamma,\delta) = \tilde O(\poly(2^L, \gamma^{-1}))\cdot \epsilon^{-4}\cdot \log(1/\delta)
\end{align*}
such that, if $m \geq m^*(\epsilon, L, \gamma,\delta)$ and $n \geq n^*(\epsilon, L, \gamma,\delta)$, then with probability at least $1 - \delta$, Algorithm \ref{alg:gdrandominit} with step size $\eta = O(4^{-L} L^{-3} \gamma^2 m^{-1})$ finds a point $\Wb^{(k)}$ that satisfies 
$$
\PP_{(\xb,y)\sim \cD} \big[ y \cdot f_{\Wb^{(k)}}(\xb) > 0 \big] \geq 1-\epsilon
$$
within $K = \tilde O(\poly(2^L,\gamma^{-1})) \cdot \epsilon^{-2}$ iterations.
\end{corollary}

We now introduce another data distribution assumption which has been made in \citet{daniely2017sgd}. 
\begin{assumption}[Separable by Conjugate Kernel]\label{assump:conjugatekernelassumption}
The conjugate kernel of fully connected neural networks is defined recursively as
\begin{align*}
    \kappa^{(0)}(\xb,\xb') = \la \xb, \xb' \ra,~\kappa^{(l+1)}(\xb,\xb') = \EE_{f\sim N(\mathbf{0}, \kappa^{(l)})} [\sigma(f(\xb))\sigma(f(\xb'))].
\end{align*}
We assume that there exists a function $f$ in the reproducing kernel Hilbert space (RKHS) $\cH$ induced by the conjugate kernel function $\kappa^{(L-1)}(\cdot,\cdot)$ with $\| f\|_{\cH} \leq 1$ such that $y_i\cdot f(\xb_i) \geq \gamma > 0$.
\end{assumption}

Under Assumption~\ref{assump:conjugatekernelassumption}, we have the following result.

\begin{corollary}\label{col:conjugatekernelgeneralization}
Under Assumption~\ref{assump:conjugatekernelassumption}, for any $\epsilon ,\delta > 0$, there exist 
\begin{align*}
    &m^*(\epsilon, L, \gamma,\delta) = \tilde O(\poly(L, \gamma^{-1}))\cdot \epsilon^{-14}\cdot \log(1/\delta),\\
    &n^*(\epsilon, L, \gamma,\delta) = \tilde O(\poly(L, \gamma^{-1}))\cdot \epsilon^{-4}\cdot \log(1/\delta)
\end{align*}
such that, if $m \geq m^*(\epsilon, L, \gamma,\delta)$ and $n \geq n^*(\epsilon, L, \gamma,\delta)$, then with probability at least $1 - \delta$, Algorithm \ref{alg:gdrandominit} with step size $\eta = O( L^{-3} \gamma^2 m^{-1})$ finds a point $\Wb^{(k)}$ that satisfies 
$$
\PP_{(\xb,y)\sim \cD} \big[ y \cdot f_{\Wb^{(k)}}(\xb) > 0 \big] \geq 1-\epsilon
$$
within $K = \tilde O(\poly(\gamma^{-1})) \cdot \epsilon^{-2}$ iterations.
\end{corollary}
\begin{remark}
Corollary~\ref{col:conjugatekernelgeneralization} shows that under Assumption~\ref{assump:conjugatekernelassumption}, a neural network trained by gradient descent can achieve $\epsilon$-expected error given $\tilde O(\epsilon^{-4})$ training examples. We remark that although \citet{daniely2017sgd} studied the same assumption, our result is not a re-derivation of the results given by \citet{daniely2017sgd}, because while they considered one-pass SGD and square loss, while we consider GD and cross-entropy loss. 
More importantly, Assumption~\ref{assump:conjugatekernelassumption} is just one specific setting our result can cover, and therefore Corollary~\ref{col:conjugatekernelgeneralization} demonstrates the power of our general theory. 
\end{remark}

\begin{remark}
A follow-up work \citet{cao2019generalizationsgd} studied the generalization performance of over-parameterized neural networks trained with one-pass SGD, and relate the generalization bound to the neural tangent kernel function studied in recent work \citep{jacot2018neural}. We remark that their generalization bound is based on an online-to-batch conversion argument, which cannot be applied to the standard gradient descent algorithm we study in this paper. Therefore our result and their result are not directly comparible.
\end{remark}

\section{Proof of the Main Theory}\label{section:proofofmaintheorem}
In this section we provide the proofs of the main results given in Section~\ref{section:maintheory}. The omitted proof can be found in the supplementary material.


\subsection{Proof of Theorem~\ref{thm:empiricalell'2populationerror}}
Here we provide the detailed proof of Theorem~\ref{thm:empiricalell'2populationerror}. We first present the lemma below, which gives an upper bound on the gradients of $L_S(\Wb)$, and relates the gradients with the empirical surrogate error $\cE_S(\Wb)$. 
\begin{lemma}\label{lemma:gradientupperbound}
For any $\delta > 0$, if
\begin{align*}
    m\geq \overline{C}\max\{ L^2\log(mn/\delta), L^{-8/3}\tau^{-4/3}\log[m/(\tau\delta)] \},~~\tau \leq \underline{C} L^{-6} [\log(m)]^{-3/2}
\end{align*}
for some large enough absolute constant $\overline{C}$ and small enough absolute constant $\underline{C}$, then with probability at least $1 - \delta$, for all $\Wb \in \cW_{\tau}$ and $l\in[L]$,
\begin{align*}
\big\|\nabla_{\Wb_l}f_{\Wb^{(0)}}(\bx_i) \big\|_F \leq C \sqrt{m},~~\big\|\nabla_{\Wb_{l}} L_S( \Wb)\big\|_F \leq C \sqrt{m} \cdot \cE_S(\Wb),
\end{align*}
where $C$ is an absolute constant.
\end{lemma}

Lemma~\ref{lemma:semismoothness} below reveals the fact that near initialization, the neural network function is \textit{almost linear} in terms of its weight parameters. As a consequence, the empirical loss function $L_S(\Wb)$ is \textit{almost smooth} in a small neighborhood around $\Wb^{(0)}$.

\begin{lemma}\label{lemma:semismoothness}
For any $\delta > 0$, if
\begin{align*}
    m\geq \overline{C}\max\{ L^2\log(mn/\delta), L^{-8/3}\tau^{-4/3}\log[m/(\tau\delta)] \},~~\tau \leq \underline{C} L^{-6} [\log(m)]^{-3/2}
\end{align*}
for some large enough absolute constant $\overline{C}$ and small enough absolute constant $\underline{C}$, then there exists an absolute constant $C$ such that with probability at least $1 - \delta$, for all $\tilde\Wb,\hat\Wb \in \cW_{\tau}$, 
\begin{align*}
    \big|f_{\hat\Wb}(\bx_i) - F_{\tilde\Wb,\hat\Wb}(\bx_i)\big| \leq 
    C L^{2} \tau^{1/3}\sqrt{m\log(m)} \cdot \sum_{l=1}^L \big\| \hat\Wb_{l} - \tilde\Wb_{l} \big\|_2 ,
\end{align*}
where
\begin{align*}
    F_{\tilde\Wb,\hat\Wb}(\xb) = f_{\tilde\Wb}(\xb) + \sum_{l=1}^L \Tr \big[ (\hat\Wb_l - \tilde\Wb_l)^\top \nabla_{\Wb_l}f_{\tilde\Wb}(\xb) \big],
\end{align*}
and
\begin{align*}
L_S(\hat\Wb) - L_S(\tilde\Wb) 
&\leq C \sum_{l=1}^L L^{2} \tau^{1/3}\sqrt{m\log(m)} \cdot \big\| \hat\Wb_{l} - \tilde\Wb_{l} \big\|_2  \cdot \cE_S(\tilde{\Wb}) \\
&\quad + \sum_{l=1}^L  \Tr [ (\hat\Wb_l - \tilde\Wb_l)^\top \nabla_{W_l}L_S(\tilde\Wb) ] + C\sum_{l=1}^L m L^3 \cdot  \| \hat\Wb_{l} - \tilde\Wb_{l} \|_2^2.
\end{align*}
\end{lemma}

\begin{proof}[Proof of Theorem~\ref{thm:empiricalell'2populationerror}]
Let $\cF_{\tau} = \{ f_{\Wb}(\xb): \Wb \in \cW_{\tau} \}$. We consider the empirical Rademacher complexity \citep{bartlett2002rademacher,mohri2018foundations,shalev2014understanding} of $\cF_{\tau}$ defined as follows
\begin{align*}
    \hat{\mathfrak{R}}_n[\cF_{\tau}] 
    = \EE_{\bxi} \Bigg[ \sup_{ \Wb \in \cW_{\tau}} \frac{1}{n} \sum_{i=1}^n \xi_i f_\Wb(\bx_i) \Bigg],
\end{align*}
where $\bxi = (\xi_1,\ldots,\xi_n)^\top$ is an $n$-dimensional vector consisting of independent Rademacher random variables $\xi_1,\ldots,\xi_n$. Since $y\in \{+1,\-1\}$, $|\ell'(z)| \leq 1$ and $\ell'(z)$ is $1$-Lipschitz continuous, by symmetrization and the standard uniform convergence results in terms of empirical Rademacher complexity \citep{mohri2018foundations,shalev2014understanding}, with probability at least $1 - \delta$ we have
\begin{align*}
    \sup_{ \Wb \in \cW_{\tau} } | \cE_S(\Wb) - \cE_\cD(\Wb) |
    &= \sup_{ \Wb \in \cW_{\tau} } \Bigg| \frac{1}{n} \sum_{i=1}^{n} \ell'\big[y_i \cdot f_{\Wb}(\bx_i) \big] - \EE_{(\xb,y)\sim \cD}  \ell'\big[y \cdot f_{\Wb}(\xb) \big]  \Bigg|\\
    &\leq  2\hat{\mathfrak{R}}_n[\cF_{\tau}] + C_1 \sqrt{\frac{\log(1/\delta)}{n}}, 
\end{align*}
where $C_1$ is an absolute constant. We now bound the term $\hat{\mathfrak{R}}_n[\cF_{\tau}]$. By definition, we have
\begin{align}
    \hat{\mathfrak{R}}_n[\cF_{\tau}] \leq I_1 + I_2,\label{eq:uniformconcentration_eq1}
\end{align}
where 
\begin{align*}
    I_1 = \EE_{\bxi} \Bigg\{ \sup_{ \Wb \in \cW_{\tau}} \frac{1}{n} \sum_{i=1}^n \xi_i \big[ f_{\Wb}(\bx_i) - F_{\Wb^{(0)},\Wb}(\bx_i)\big] \Bigg\},~~I_2 = \EE_{\bxi} \Bigg\{ \sup_{ \Wb \in \cW_{\tau}} \frac{1}{n} \sum_{i=1}^n \xi_i  F_{\Wb^{(0)},\Wb}(\bx_i) \Bigg\},
\end{align*}
and
\begin{align*}
    F_{\Wb^{(0)},\Wb}(\xb) = & \sum_{l=1}^L \Tr \big[ (\Wb_l - \Wb^{(0)}_l)^\top \nabla_{\Wb_l}f_{\Wb^{(0)}}(\xb) \big] + f_{\Wb^{(0)}}(\xb).
\end{align*} 
For $I_1$, by Lemma~\ref{lemma:semismoothness}, we have
\begin{align*}
    I_1&\leq \max_{i\in[n]} \big|f_{\Wb}(\bx_i) - F_{\Wb^{(0)},\Wb}(\bx_i)\big| \leq C_2L^4\sqrt{m\log(m)} \tau^{4/3}
\end{align*}
for all $i\in [n]$, where $C_2$ is an absolute constant. 
For $I_2$, note that $\EE_{\bxi} \big\{ \sup_{ \Wb \in \cW_{\tau}} \sum_{i=1}^n \xi_i  f_{\Wb^{(0)}}(\bx_i) \big\} = 0$, and therefore
\begin{align*}
    I_2 &= \frac{1}{n}  \sum_{l=1}^L \EE_{\bxi} \Bigg\{  \sup_{ \|\tilde\Wb_l\|_F \leq \tau }  \Tr \Bigg[ \tilde\Wb_l^\top \sum_{i=1}^n \xi_i \nabla_{\Wb_l}f_{\Wb^{(0)}}(\bx_i) \Bigg] \Bigg\} \leq  \frac{\tau}{n}  \sum_{l=1}^L \EE_{\bxi} \Bigg[ \Bigg\| \sum_{i=1}^n \xi_i \nabla_{\Wb_l}f_{\Wb^{(0)}}(\bx_i) \Bigg\|_F \Bigg].
\end{align*}
By Jensen's inequality, 
\begin{align*}
    I_2 \leq \frac{\tau}{n} \sum_{l=1}^L \sqrt{ \EE_{\bxi} \Bigg[ \Bigg\| \sum_{i=1}^n \xi_i \nabla_{\Wb_l}f_{\Wb^{(0)}}(\bx_i) \Bigg\|_F^2 \Bigg]} = \frac{\tau}{n} \sum_{l=1}^L \sqrt{ \sum_{i=1}^n \big\| \nabla_{\Wb_l}f_{\Wb^{(0)}}(\bx_i) \big\|_F^2 }.
\end{align*}
Now by Lemma~\ref{lemma:gradientupperbound}, we have
$\big\|\nabla_{\Wb_l}f_{\Wb^{(0)}}(\bx_i) \big\|_F \leq C_3  \sqrt{m}$ for all $l\in[L]$, where $C_3$ is an absolute constant. Therefore $I_2 \leq C_3 L\tau\cdot \sqrt{m/n}$. 
Plugging  in the bounds of $I_1$ and $I_2$ into \eqref{eq:uniformconcentration_eq1} and applying Markov's inequality 
\begin{align*}
    \EE_{(\xb,y)\sim \cD} \big\{- \ell'\big[y \cdot f_{\Wb}(\xb)\big]\big\} \geq \PP_{(\xb,y)\sim \cD}\big\{- \ell'\big[y \cdot f_{\Wb}(\xb)\big] \geq 1/2\big\} / 2 = \PP_{(\xb,y)\sim \cD}\big[y \cdot f_{\Wb}(\xb) < 0\big] / 2
\end{align*}
completes the proof.
\end{proof}

\subsection{Proof of Theorem~\ref{thm:convergence_gd_empirical}}



The following lemma is given by \citet{zou2018stochastic}, which gives a bound on the neural network output at initialization. 

\begin{lemma}[\citet{zou2019gradient}]\label{lemma:randinit_outputbounded}
For any $\delta > 0$, with probability at least $1 - \delta$, $ |f_{\Wb^{(0)}} (\bx_i)| \leq C\sqrt{\log( n / \delta)} $
for all $i\in[n]$, where $C$ is an absolute constant.
\end{lemma}

\begin{proof}[Proof of Theorem~\ref{thm:convergence_gd_empirical}]
Set $\tau = \tilde O(  B^{-1}\epsilon^{-1} m^{-1/2} )$. Then there exist $\eta = O(L^{-3}B^2 m^{-1})$, $K = \tilde O(L^3B^{-4} \epsilon^{-2})$ and $m^* = \tilde O( L^{12}B^{-4}\epsilon^{-2} ) \cdot \log(1/\delta)$
such that when $m \geq m^*$, all assumptions of  Lemmas~\ref{lemma:gradientupperbound}, \ref{lemma:semismoothness} hold, and 
\begin{align}
    &(K\eta)^{1/2} B^{-1} [\log(n/\delta)]^{1/2} \leq \nu\tau,\label{eq:convergence_gd_eq1.1}\\
    &L^3m \eta \leq \nu B^2,\label{eq:convergence_gd_eq1.2}\\
    &L^2 \tau^{1/3} [m\log(m)]^{1/2} \leq \nu B^2 m, \label{eq:convergence_gd_eq1.3}\\
    &(K\eta \cdot m)^{-1/2} B^{-1} \leq \epsilon \label{eq:convergence_gd_eq1.4}
\end{align}
for some small enough absolute constant $\nu$. We now prove by induction that $\Wb^{(k)} \in \cW(\Wb^{(0)},\tau/2)$, $k\in \{0\}\cup [K]$. By definition clearly we have $\Wb^{(0)} \in \cW(\Wb^{(0)},\tau/2)$. Suppose that $\Wb^{(k)} \in \cW(\Wb^{(0)},\tau/2)$ for all $k = 0,\ldots, t$. Then for all $l\in [L]$ we have 
\begin{align*}
    \| \Wb_l^{(t+1)} - \Wb_l^{(0)} \|_F \leq \| \Wb_l^{(t)} - \Wb_l^{(0)} \|_F + \eta \| \nabla_{\Wb_l} L_S(\Wb^{(t)}) \|_F \leq \tau/2 + \tau/2 = \tau, 
\end{align*}
where the last inequality follows by  Lemma~\ref{lemma:gradientupperbound} and the definition of $\tau$ and $\eta$ (note that a comparison between \eqref{eq:gradientlowerbound_assump} and Lemma~\ref{lemma:gradientupperbound} implies that $B = O(1)$).
Therefore $\Wb^{(t+1)} \in \cW_{\tau}$. Plugging in the gradient upper bound given by Lemma~\ref{lemma:gradientupperbound} and assumption \eqref{eq:gradientlowerbound_assump} into the result of Lemma~\ref{lemma:semismoothness}, we obtain
\begin{align*}
    L_S(\Wb^{(k+1)}) - L_S(\Wb^{(k)}) \leq  C_1 \sum_{l=1}^L \big[ L^{2} \tau^{1/3}\eta m\sqrt{\log(m)}  + L^3 m^2 \eta^2 \big] \cdot \cE_S^2(\Wb^{(k)})  - \eta \cdot B^2 m \cdot \cE_S^2(\Wb^{(k)})
\end{align*}
for all $k = 0,\ldots, t$, where $C_1,C_2$ are absolute constants. Plugging in the bounds \eqref{eq:convergence_gd_eq1.2} and \eqref{eq:convergence_gd_eq1.3}, we have
\begin{align}\label{eq:convergence_gd_eq2}
    L_S(\Wb^{(k+1)}) - L_S(\Wb^{(k)}) \leq -  \eta B^2 m \cE_S^2(\Wb^{(k)})/2
\end{align}
for all $k = 0,\ldots, t$. Combining \eqref{eq:convergence_gd_eq2} with Lemma~\ref{lemma:gradientupperbound} gives
\begin{align*}
    \big\|\nabla_{\Wb_{l}} L_S(\Wb^{(k)})\big\|_F \leq C_2 \eta^{-1/2} B^{-1} [ L_S(\Wb^{(k)}) -  L_S(\Wb^{(k + 1)}) ]^{1/2}
\end{align*}
for all $k=0,\ldots,t$, where $C_2$ is an absolute constant. Note that by Lemma~\ref{lemma:randinit_outputbounded} and the fact that $\ell(z) \leq 1 + |z|$, we have $ L_S(\Wb^{(0)}) - L_S(\Wb^{(K)}) \leq C_3 [\log(n/\delta)]^{1/2}$ for some absolute constant $C_3$. Therefore by Jensen's inequality,
\begin{align*}
    \| \Wb_l^{(t+1)} - \Wb_l^{(0)} \|_F 
    & \leq \eta \sum_{k=0}^t\big\|\nabla_{\Wb_{l}} L_S(\Wb^{(k)})\big\|_F\\
    &\quad\leq C_2 \eta^{1/2} B^{-1} \sum_{k=0}^t  [ L_S(\Wb^{(k)}) -  L_S(\Wb^{(k + 1)}) ]^{1/2}\\
    &\quad\leq C_2 \sqrt{K\eta} B^{-1} \cdot [ L_S(\Wb^{(0)}) - L_S(\Wb^{(k)}) ]^{1/2}\\
    &\quad\leq C_4 \sqrt{K\eta} B^{-1} [\log(n/\delta)]^{1/2}\\
    &\quad\leq \tau/2,
\end{align*}
where $C_4$ is an absolute constant, and the last inequality follows by \eqref{eq:convergence_gd_eq1.1}. Therefore by induction, $\Wb^{(k)} \in \cW(\Wb^{(0)},\tau/2)$ for all $k\in [K]$. This also implies that \eqref{eq:convergence_gd_eq2} holds for all $k= 0 ,\ldots, K - 1$.  Let $k^* = \argmin_{k\in \{0,\ldots,K-1\}} \cE_S(\Wb^{(k)}) $. Telescoping over $k$ gives
\begin{align*}
    L_S(\Wb^{(K)}) - L_S(\Wb^{(0)})
    \leq - K\eta B^2 m \cdot \cE_S^2(\Wb^{(k^*)}).
\end{align*}
Hence by \eqref{eq:convergence_gd_eq1.4} we have
\begin{align*}
    \cE_S(\Wb^{(k^*)}) \leq (K\eta \cdot m)^{-1/2} B^{-1} \leq \epsilon,
\end{align*}
This completes the proof.
\end{proof}

\subsection{Proof of Corollary~\ref{col:randomfeaturegeneralization}}
In this section we give the proof of Corollary~\ref{col:randomfeaturegeneralization}. 
The following lemma verifies that under Assumption~\ref{assump:nonlinearseparable}, \eqref{eq:gradientlowerbound_assump} indeed holds with $B$ independent in both $m$ and $n$. 
\begin{lemma}\label{lemma:gradientlowerbound_randomfeature}
For any $\delta >0$, if
\begin{align*}
    m\geq \overline{C} \cdot \max\{ 4^L  L^2\gamma^{-2} \log(mnL/\delta), L^{-8/3}\tau^{-4/3}\log[m/(\tau\delta)] \},~~\tau \leq \underline{C} \cdot 8^{-L} L^{-2} \gamma^3 [\log(m)]^{-3/2}
\end{align*}
for some large enough absolute constant $\overline{C}$ and small enough absolute constant $\underline{C}$, then with probability at least $1 - \delta$, there exists an absolute constant $C$ such that
\begin{align*}
    \big\| \nabla_{\Wb_{L}} L_S( \Wb) \big\|_F \geq C \cdot 2^{-L}\cdot \gamma \sqrt{m} \cdot \cE_S(\Wb)
\end{align*}
for all $\Wb \in \cW_{\tau}$.
\end{lemma}
\begin{proof}[Proof of Corollary~\ref{col:randomfeaturegeneralization}] 
Corollary~\ref{col:randomfeaturegeneralization} directly follows by plugging in $B = O(2^{-L} \gamma)$ given by Lemma~\ref{lemma:gradientlowerbound_randomfeature} and the assumptions $m \geq \tilde O(L^{24} 2^{8L} \gamma^{-8})\cdot \epsilon^{-14}$, $n\geq \tilde O(L^4 4^{L} \gamma^{-2})\cdot \epsilon^{-4}$ into Corollary~\ref{col:populationerrorwithgradientlowerbound}. 
\end{proof}

\subsection{Proof of Corollary~\ref{col:conjugatekernelgeneralization}}
In this section we give the proof of Corollary~\ref{col:conjugatekernelgeneralization}. Similar to the proof of Corollary~\ref{col:randomfeaturegeneralization}, we mainly need to derive a gradient lower bound of the form \eqref{eq:gradientlowerbound_assump}.
The result is given in the following lemma, which gives a similar result in part of the proof of Claim 1 in \citet{daniely2017sgd}.
\begin{lemma}\label{lemma:gradientlowerbound_conjugatekernel}
For any $\delta >0$, if
\begin{align*}
    m\geq \overline{C} \cdot \max\{ \gamma^{-2} \log(mn/\delta), \tau^{-4/3}\log[m/(\tau\delta)] \},~~\tau \leq \underline{C} \cdot \gamma^3 [\log(m)]^{-3/2}
\end{align*}
for some large enough absolute constant $\overline{C}$ and small enough absolute constant $\underline{C}$, then with probability at least $1 - \delta$, there exists an absolute constant $C$ such that
\begin{align*}
    \big\| \nabla_{\Wb_{L}} L_S( \Wb) \big\|_F \geq C \gamma \sqrt{m} \cdot \cE_S(\Wb)
\end{align*}
for all $\Wb \in \cW_{\tau}$.
\end{lemma}
\begin{proof}[Proof of Corollary~\ref{col:conjugatekernelgeneralization}] 
Corollary~\ref{col:conjugatekernelgeneralization} directly follows by plugging in $B = O(\gamma)$ given by Lemma~\ref{lemma:gradientlowerbound_conjugatekernel} and the assumptions $m \geq \tilde O(L^{24} \gamma^{-8})\cdot \epsilon^{-14}$, $n\geq \tilde O(L^4 \gamma^{-2})\cdot \epsilon^{-4}$  into Corollary~\ref{col:populationerrorwithgradientlowerbound}. 
\end{proof}

\section{Conclusions and Future Work}\label{section:conclusion}
In this paper, we provided a generalization guarantee of gradient descent for training deep ReLU networks under over-parameterization, which hold under mild data distribution assumptions. Although we only focus on gradient descent and cross-entropy loss for binary classification, our results can be extended to stochastic gradient descent, other loss functions and multi-class classification. 
In addition, 
we will derive generalization bounds for deep learning based on the ``small-ball" assumption proposed in \cite{mendelson2014learning}. Another interesting direction is to investigate the generalization of gradient descent using stability-based analysis \citep{hardt2015train}.

\section*{Acknowledgements}
 This research was sponsored in part by the National Science Foundation CAREER Award IIS-1906169, IIS-1903202, and Salesforce Deep Learning Research Award. The views and conclusions contained in this paper are those of the authors and should not be interpreted as representing any funding agencies.

\appendix
\onecolumn

\section{Matrix Product Representation for Deep ReLU Networks}\label{sec:matprodrepresentation}
Here we introduce the matrix product representation for deep ReLU networks, which plays a pivotal role in our analysis. Given parameter matrices $\Wb_1,\ldots,\Wb_L$ and an input $\xb$, we denote by $\xb_l$ the output of the $l$-th layer of the ReLU network, and set $\xb_0 = \xb$. 
We also define diagonal binary matrices $\bSigma_l(\xb) = \Diag( \ind\{ \wb_{l,1}^{\top} \xb_{l-1} > 0 \}, \ldots, \ind\{ \wb_{l,m_l}^{\top} \xb_{l-1} > 0 \} )$, $l\in [L]$.
Then we have the following representations for the neural network function and its gradients:
\begin{align*}
    &f_{\Wb}(\xb) = \vb^\top\Bigg[\prod_{r=1}^L\bSigma_{r}(\xb)\Wb_{r}^{\top}\Bigg]\xb,\\
    &\nabla_{\Wb_l} f_{\Wb}(\xb) =\xb_{l-1}\vb^\top \Bigg[\prod_{r=l+1}^L\bSigma_{r}(\xb)\Wb_{r}^\top\Bigg]\bSigma_{l}(\xb), l\in [L],
\end{align*}
where we use the following matrix product notation:
\begin{align*}
    \prod_{r = l_1}^{l_2} \Ab_r :=\left\{
    \begin{array}{ll}
        \Ab_{l_2}\Ab_{l_2-1} \cdots  \Ab_{l_1} & \text{if }l_1\leq l_2 \\
        \Ib & \text{otherwise.}
    \end{array}
    \right.
\end{align*}
Since this paper studies the generalization performance of neural network learning, we frequently need to study the training examples $(\bx_1,y_1),\ldots,(\bx_n,y_n)$ as well as a test example $(\xb,y)\sim \cD$. To distinguish the $i$-th example in the training sample and the $l$-th layer output of the test input $\xb$, we use the following notations:
\begin{itemize}
    \item For $i=1,\ldots,n$, $l=1,\ldots,L$, we use $\bx_i$ to denote the $i$-th training input, and $\bx_{l,i}$ the output of the $l$-th layer with input $\bx_{i}$.
    \item For $l=1,\ldots,L$, we denote by $\xb_l$ the output of the $l$-th layer with test input $\xb$.
\end{itemize}

\section{Proof of Main Results in Section~\ref{section:proofofmaintheorem}}\label{section:appendixA}

In this section we provide proofs of theorems and lemmas given in Section~\ref{section:proofofmaintheorem}.

\subsection{Proof of Lemma~\ref{lemma:gradientupperbound}}

Before we prove Lemma~\ref{lemma:gradientupperbound}, we need the following technical lemma, which is a simplified version of Theorem~5.3 given in \citet{zou2019gradient}. It characterizes several basic scaling properties of deep ReLU networks around random initialization. Here we use the following extension of the notations introduced in Section~\ref{sec:matprodrepresentation}: we denote by $\tilde\bx_{l,i}$ and $\hat\bx_{l,i}$ the hidden outputs of the ReLU network with input $\bx_i$ and weights $\tilde\Wb$, $\hat\Wb$ respectively. Similar notations are also used for the binary diagonal matrices $\tilde\bSigma_l(\bx_i)$ and $\hat\bSigma_l(\bx_i)$.

\begin{lemma}[\citet{zou2019gradient}]\label{lemma:scaling_perturbation} 
There exist absolute constants $\overline{C}, \underline{C}$ such that, for any $\delta > 0$, if
$$
m\geq \overline{C}\max\{ L^2\log(mn/\delta), L^{-8/3}\tau^{-4/3}\log[m/(\tau\delta)] \},~~\tau \leq \underline{C} L^{-5} [\log(m)]^{-3/2},
$$
then with probability at least $1 - \delta$, the following results hold uniformly for all $\hat\Wb, \tilde\Wb \in \cW_{\tau}$:
\begin{enumerate}[label=(\roman*)]
    \item $\|\tilde \Wb_l\|_2, \| \tilde\bx_{l,i} \|_2 \le \overline{C}$ for all $l\in[L]$ and $i\in[n]$.\label{item:bound_tilde_W,x}
    \item $ \big\| \prod_{r=l}^{L}\tilde \bSigma_{r}(\bx_i)\tilde \Wb_r^\top  \big\|_2 \leq \overline{C} L $ for all $l\in[L] $ and $i\in[n]$.\label{item:matrixnorm_middlelayer}
    \item $\| \hat \bx_{l,i} - \tilde \bx_{l,i} \|_2\le  \overline{C} \sqrt{L} \cdot \sum_{r=1}^l \| \hat{\Wb}_{r} - \tilde{\Wb}_{r} \|_2$ for all $l\in[L]$ and $i\in[n]$.\label{item:difference_xbound}
    \item $ \| \hat{\bSigma}_{l}(\bx_i) - \tilde{\bSigma}_{l}(\bx_i) \|_0 \leq \overline{C} L^{4/3}\tau^{2/3} m$ for all $l\in[L]$ and $i\in[n]$.\label{item:difference_sparsity}
    \item $ \vb^\top \big[\prod_{r=l}^{L}\tilde \bSigma_{r}(\bx_i)\tilde \Wb_r^\top \big] \ab \leq \overline{C} L^{2/3} \tau^{1/3}\sqrt{m\log(m)} $ for all $l \in [L] $, $i\in[n]$ and all $\ab\in \RR^{m_{l - 1}}$ satisfying $\| \ab \|_2 = 1$, $\| \ab \|_0 \leq \overline{C} L^{4/3}\tau^{2/3} m$.
    \label{item:matrixnorm_lastlayer}
    \item $ \big\| \vb^\top \big[\prod_{r=l}^{L}\tilde \bSigma_{r}(\bx_i)\tilde \Wb_r^\top \big] \big\|_2  \leq \overline{C} \sqrt{m} $ for all $l \in [L] $ and $i\in[n]$.
    \label{item:matrixnorm_lastlayernosparse}
\end{enumerate}
\end{lemma}

We are now ready to give the proof of Lemma~\ref{lemma:gradientupperbound}. 

\begin{proof}[Proof of Lemma~\ref{lemma:gradientupperbound}] Denote $\tilde{y}_i = f_{\Wb}(\bx_i)$. Then by definition, we have
\begin{align*}
    \nabla_{\Wb_{l}} L_S(\Wb) = \frac{1}{n} \sum_{i=1}^n \ell'(y_i\tilde y_i)\cdot y_i\cdot \bx_{l-1,i}\vb^\top\Bigg(\prod_{r=l+1}^L\bSigma_{r}(\bx_i)\Wb_r^{\top}\Bigg)\bSigma_{l}(\bx_i).
\end{align*}
By \ref{item:bound_tilde_W,x} and \ref{item:matrixnorm_lastlayernosparse} in Lemma~\ref{lemma:scaling_perturbation}, we have 
\begin{align*}
    \| \bx_{l-1,i} \|_2 ,~\Bigg\| \vb^\top \prod_{r=l+1}^L\bSigma_{r}(\bx_i)\Wb_l^{\top} \Bigg\|_2 \leq C_1 \sqrt{m},
\end{align*}
for all $i\in [n]$ and $l\in [L]$, where $C_1$ is an absolute constant. By the triangle inequality, we have
\begin{align*}
    \| \nabla_{\Wb_{l}} L_S( \Wb) \|_F &\leq  \frac{1}{n} \sum_{i=1}^n \Bigg\| \ell'(y_i\tilde y_i)\cdot y_i\cdot \bx_{l-1,i}\vb^\top\Bigg(\prod_{r=l+1}^L\bSigma_{r}(\bx_i)\Wb_l^{\top}\Bigg)\bSigma_{l}(\bx_i) \Bigg\|_F\\
     &=  \frac{1}{n} \sum_{i=1}^n \big\|  \ell'(y_i\tilde y_i)\cdot y_i\cdot \bx_{l-1,i} \big\|_2 \cdot \Bigg\| \vb^\top\Bigg(\prod_{r=l+1}^L\bSigma_{r}(\bx_i)\Wb_l^{\top}\Bigg)\bSigma_{l}(\bx_i) \Bigg\|_2\\
     &\leq C_2 \sqrt{m} \cdot \Bigg[- \frac{1}{n} \sum_{i=1}^n \ell'(y_i\tilde y_i) \Bigg]
\end{align*}
for all $l\in [L]$, where $C_2$ is an absolute constant. This completes the proof.
\end{proof}

\subsection{Proof of Lemma~\ref{lemma:semismoothness}}
We now present the proof of Lemma~\ref{lemma:semismoothness}.

\begin{proof}[Proof of Lemma~\ref{lemma:semismoothness}]
For $i\in[n]$, denote by $\hat y_i$, $\tilde y_i$ the outputs of the network with input $\bx_i$ and parameter matrices $\hat\Wb$, $\tilde\Wb$ respectively. Then we have
\begin{align*}
    f_{\hat\Wb}(\bx_i) - f_{\tilde\Wb}(\bx_i)& = \vb^\top\Bigg[ \prod_{l=1}^L \hat\bSigma_{l}(\bx_i)\hat\Wb_{l}^{\top}\Bigg]\bx_i -\vb^\top\Bigg[ \prod_{l=1}^L \tilde\bSigma_{l}(\bx_i)\tilde\Wb_{l}^{\top}\Bigg]\bx_i\\
    & = \sum_{l=1}^L \Bigg[\prod_{r=l+1}^L\tilde\bSigma_{r}(\bx_i)\tilde\Wb_{r}^{\top}\Bigg]\big[\hat\bSigma_{l}(\bx_i)\hat\Wb_l^{\top} - \tilde\bSigma_{l}(\bx_i)\tilde\Wb_{l}^{\top}\big]\hat\bx_{l-1,i},
\end{align*}
and therefore $ f_{\hat\Wb}(\bx_i) - f_{\tilde\Wb}(\bx_i) = I_{1} + I_{2} + I_{3}$, where
\begin{align*}
    & I_{1} = \sum_{l=1}^L \vb^\top\Bigg[\prod_{r=l+1}^L\tilde\bSigma_{r}(\bx_i)\tilde\Wb_{r}^{\top}\Bigg]\big[\hat\bSigma_{l}(\bx_i) - \tilde\bSigma_{l}(\bx_i)\big]\hat\Wb_{l}^{\top}\hat\bx_{l-1,i},\\
    & I_{2} = \sum_{l=1}^L \vb^\top\Bigg[\prod_{r=l+1}^L\tilde\bSigma_{r}(\bx_i) \tilde\Wb_{r}^{\top}\Bigg] \tilde\bSigma_{l}(\bx_i)\big( \hat\Wb_{l} - \tilde\Wb_{l} \big)^\top(\hat\bx_{l-1,i}-\tilde\bx_{l-1,i}),\\
    & I_{3} = \sum_{l=1}^L \vb^\top \Bigg[\prod_{r=l+1}^L\tilde\bSigma_{r}(\bx_i)\tilde\Wb_{r}^{\top}\Bigg] \tilde\bSigma_{l}(\bx_i)\big(\hat\Wb_{l} - \tilde\Wb_{l}\big)^\top\tilde\bx_{l-1,i}.
\end{align*}
For $I_{1}$, note that by Lemma~\ref{lemma:scaling_perturbation}, for any $l=1,\ldots,L$ we have
\begin{align*}
    \big\| \big[\hat\bSigma_{l}(\bx_i) - \tilde\bSigma_{l}(\bx_i)\big]\hat\Wb_{l}^{\top}\hat\bx_{l-1,i} \big\|_2 &\leq \big\| \hat\Wb_{l}^{\top}\hat\bx_{l-1,i} - \tilde\Wb_{l}^{\top}\tilde\bx_{l-1,i} \big\|_2\\
    & \leq \big\| (\hat\Wb_{l}^{\top} - \tilde\Wb_{l}^{\top})\hat\bx_{l-1,i} \big\|_2 + \big\| \tilde\Wb_{l}^{\top}(\hat\bx_{l-1,i} - \tilde\bx_{l-1,i}) \big\|_2\\
    & \leq C_1 \big\| \hat\Wb_{l} - \tilde\Wb_{l} \big\|_2 + C_1 L \sum_{l=1}^L \big\| \hat\Wb_{l} - \tilde\Wb_{l} \big\|_2 \\
    &\leq C_2 L\cdot \sum_{l=1}^L \big\| \hat\Wb_{l} - \tilde\Wb_{l} \big\|_2,
\end{align*}
where the first inequality follows by checking the non-zero diagonal entries of $\hat\bSigma_{l}(\bx_i) - \tilde\bSigma_{l}(\bx_i)$, and $C_1,C_2$ are absolute constants. Therefore by Lemma~\ref{lemma:scaling_perturbation} we have
\begin{align*}
    |I_{1}| &\leq \sum_{l=1}^L  \Bigg\| \vb^\top\Bigg[\prod_{r=l+1}^L\tilde\bSigma_{r}(\bx_i)\tilde\Wb_{r}^{\top}\Bigg]\cdot \big|\hat\bSigma_{l}(\bx_i) - \tilde\bSigma_{l}(\bx_i)\big| \Bigg\|_2 \cdot \big\| \big[\hat\bSigma_{l}(\bx_i) - \tilde\bSigma_{l}(\bx_i)\big]\hat\Wb_{l}^{\top}\hat\bx_{l-1,i}\big\|_2 \\
    &\leq C_3 L^2 \tau^{1/3} \sqrt{m\log(m)} \cdot\sum_{l=1}^L \big\| \hat\Wb_{l} - \tilde\Wb_{l} \big\|_2,
\end{align*}
where $C_3$ is an absolute constant. 
For $I_{2}$, by Lemma~\ref{lemma:scaling_perturbation} we have
\begin{align*}
    |I_{2}| &\leq C_4\sqrt{m} \cdot L\cdot \sum_{l=1}^L \| \hat\Wb_{l} - \tilde\Wb_{l} \|_2 \cdot L\cdot \sum_{r=1}^l \| \hat\Wb_{r} - \tilde\Wb_{r} \|_2\\
    &\leq C_4L^3\cdot \sqrt{m} \cdot \sum_{l=1}^L \| \hat\Wb_{l} - \tilde\Wb_{l} \|_2^2\\
    &\leq C_4 L^{2} \tau^{1/3}\sqrt{m\log(m)} \cdot\sum_{l=1}^L \big\| \hat\Wb_{l} - \tilde\Wb_{l} \big\|_2,
\end{align*}
where $C_4$ is an absolute constant. 
For $I_{3}$, we have
\begin{align*}
    I_{3} &= \sum_{l=1}^L \vb^\top \Bigg[\prod_{r=l+1}^L\tilde\bSigma_{r}(\bx_i)\tilde\Wb_{r}^{\top}\Bigg] \tilde\bSigma_{l}(\bx_i)\big(\hat\Wb_{l} - \tilde\Wb_{l}\big)^\top\tilde\bx_{l-1,i}\\
    &= \sum_{l=1}^L \Tr\Bigg\{ \big(\hat\Wb_{l} - \tilde\Wb_{l}\big)^\top \tilde\bx_{l-1,i} \vb^\top \Bigg[\prod_{r=l+1}^L\tilde\bSigma_{r}(\bx_i)\tilde\Wb_{r}^{\top}\Bigg] \tilde\bSigma_{l}(\bx_i) \Bigg\}\\
    &=\sum_{l=1}^L \Tr\big[ (\hat\Wb_l - \tilde\Wb_l)^\top \nabla_{\Wb_l} f_{\tilde\Wb}(\bx_i) \big].
\end{align*}
Combining the bounds of $I_1$, $I_2$ and $I_3$ completes the proof of the first bound. For the second result, since $|\ell''(z)| \leq 1/2$ for all $z \in \RR$, we have
\begin{align*}
    L_S(\hat\Wb) - L_S(\tilde\Wb) = \frac{1}{n}\sum_{i=1}^n  [\ell(y_i \hat y_i) - \ell(y_i \tilde y_i)]\leq \frac{1}{n}\sum_{i=1}^n [\ell'(y_i \hat y_i) \cdot y_i\cdot (\hat y_i - \tilde y_i) + (\hat y_i - \tilde y_i)^2 / 4 ].
\end{align*}
Denote $\Delta_{i} = \hat y_{i} - \tilde y_{i}$. Then
\begin{align}\label{eq:semismoothness_empirical_eq1}
    L_S(\hat\Wb) - L_S(\tilde\Wb) \leq \frac{1}{n}\sum_{i=1}^n \big[ \ell'(y_i \tilde y_{i})\cdot y_i\cdot \Delta_{i} + (\Delta_{i})^2  / 4 \big].
\end{align}
Plugging in the bound of $\big|f_{\hat\Wb}(\bx_i) - F_{\tilde\Wb,\hat\Wb}(\bx_i)\big|$ gives
\begin{align}\label{eq:semismoothness_empirical_eq2}
\frac{1}{n}\sum_{i=1}^n  \ell'(y_i \tilde y_{i})\cdot y_i\cdot \Delta_{i} &\leq 
C_5 \sum_{l=1}^L L^{2} \tau^{1/3}\sqrt{m\log(m)} \cdot \big\| \hat\Wb_{l} - \tilde\Wb_{l} \big\|_2 \cdot \Bigg[ - \frac{1}{n}\sum_{i=1}^n  \ell'(y_i \tilde y_{i}) \Bigg]
\\
&\quad + C_5\sum_{l=1}^L L^3 \cdot  \| \hat\Wb_{l} - \tilde\Wb_{l} \|_2^2 \cdot \Bigg[ - \frac{1}{n}\sum_{i=1}^n  \ell'(y_i \tilde y_{i}) \Bigg]\nonumber
\\
&\quad + \sum_{l=1}^L  \Tr [ (\hat\Wb_l - \tilde\Wb_l)^\top \nabla_{W_l} L_S(\tilde\Wb) ],\nonumber
\end{align}
where $C_5$ is an absolute constant. Moreover, by Lemma~\ref{lemma:scaling_perturbation}, clearly we have
\begin{align}\label{eq:semismoothness_empirical_eq3}
    \Delta_i^2 = [\vb^\top(\hat\bx_{L,i} - \tilde\bx_{L,i} )]^2 \leq C_6  m L^2 \Bigg( \sum_{l=1}^L \| \hat\Wb_l - \tilde\Wb_l \|_2\Bigg)^2 \leq C_6  mL^3 \cdot \sum_{l=1}^L \| \hat\Wb_l - \tilde\Wb_l \|_2^2
\end{align}
for some absolute constant $C_6$. 
Plugging \eqref{eq:semismoothness_empirical_eq2} and \eqref{eq:semismoothness_empirical_eq3} into \eqref{eq:semismoothness_empirical_eq1}, and using the fact $ - \ell'(z) \leq 1 $, $z\in \RR$ completes the proof.
\end{proof}

\subsection{Proof of Lemma~\ref{lemma:gradientlowerbound_randomfeature}}

Here we give the proof of Lemma~\ref{lemma:gradientlowerbound_randomfeature}. Again, we extend the notations introduced in Section~\ref{sec:matprodrepresentation} by denoting $\bx_{l,i}^{(k)}$ and $\bSigma_l^{(k)}(\bx_i)$ the network hidden layer outputs and binary diagonal matrices with input $\bx_i$ and weights $\Wb^{(k)}$ respectively.

We first introduce the following two lemmas, which are based on Assumption~\ref{assump:nonlinearseparable}. Lemma~\ref{lemma:randinit_linearseparable_lastlayer} below shows that under our data distribution assumptions, the hidden layer outputs of the deep ReLU network is linearly separable with high probability. Lemma~\ref{lemma:randinit_gradientlowerbound_empirical} takes advantage of this linearly separable property and further gives a lower bound result with respect to the initialized weights, which plays an essential role in the proof of our gradient lower bound.

\begin{lemma}\label{lemma:randinit_linearseparable_lastlayer}
For any $\delta>0$, if 
$ m\geq C \cdot 4^L \cdot L^2\gamma^{-2} \log(nL/\delta)$
for some large enough absolute constant $C$, then with probability at least $1 - \delta$, there exist $\balpha_1\in S^{m_1 - 1},\ldots,\balpha_L\in S^{m_L - 1}$ such that
$y_i\cdot \la \balpha_l, \bx_{l,i}^{(0)} \ra \geq 2^{-(l+1)} \gamma $ 
for all $i\in [n]$ and $l\in [L]$. 
\end{lemma}

\begin{lemma}\label{lemma:randinit_gradientlowerbound_empirical}
For any $\delta>0$, under the same assumptions as Lemma~\ref{lemma:randinit_linearseparable_lastlayer}, with probability at least $1 - \delta$, the inequality
\begin{align*}
    \sum_{j=1}^{m_L} \Bigg\| \frac{1}{n} \sum_{i=1}^n [ a(\bx_i,y_i) \cdot y_i \cdot \sigma'( \wb_{L,j}^{(0)\top} \bx_{L-1,i}^{(0)} ) \cdot \bx_{L-1,i}^{(0)} ] \Bigg\|_2^2 \geq  4^{-L}/8\cdot m_L \cdot \gamma^2 \cdot \Bigg[ \frac{1}{n}\sum_{i=1}^n  a(\bx_i,y_i) \Bigg]^2
\end{align*}
holds for any function $a(\xb,y): S^{d-1}\times \{\pm 1\} \rightarrow \RR^+$.
\end{lemma}

We are now ready to prove Lemma~\ref{lemma:gradientlowerbound_randomfeature}.

\begin{proof}[Proof of Lemma~\ref{lemma:gradientlowerbound_randomfeature}]
For $i\in[n]$, denote by $\tilde{y}_i = f_{ \Wb}(\bx_{i})$ the output of the neural network with parameter matrices $\Wb_1,\ldots,\Wb_L$ and input $\bx_i$, and define $\Gb = (\gb_1,\ldots,\gb_{m_L}) \in \RR^{m_{L-1} \times m_L} $, where
\begin{align*}
    &\Gb = \frac{1}{n} \sum_{i=1}^n \ell'(y_i\tilde y_i) \cdot y_i \cdot \bx_{L-1,i}^{(0)} \vb^\top \bSigma_{L-1}^{(0)}(\bx_i),\\
    &\gb_j = \frac{1}{n} \sum_{i=1}^n [-\ell'(y_i\tilde y_i) ]  \cdot y_i \cdot \sigma'( \wb_{L,j}^{(0)\top} \bx_{L-1,i}^{(0)} ) \cdot \bx_{L-1,i}^{(0)},~j\in[m_L].
\end{align*}
Since $ 0 \leq | v_j \ell'(y_i\tilde y_i) | \leq 1$, by Lemma~\ref{lemma:randinit_gradientlowerbound_empirical} with $a(\bx_i,y_i) = -\ell'(y_i\tilde y_i)$, with probability at least $1 - \delta / 2$, 
\begin{align*}
    \| \Gb \|_F = \sqrt{\sum_{j=1}^{m_L} \| \gb_j \|_2^2 }\geq  \sqrt{ 4^{-L}/8\cdot m_{L} \cdot \gamma^2 \cdot \Bigg[ \frac{1}{n} \sum_{i=1}^n \ell'(y_i\tilde y_i) \Bigg] ^2 } = 2^{-L}/(2\sqrt{2}) \cdot \sqrt{m_{L}} \cdot \gamma \cdot \cE_S(\Wb).
\end{align*}
Since $\ell'(z) < 0$ and
\begin{align*}
    \nabla_{\Wb_L}L_S (\Wb) = \frac{1}{n} \sum_{i=1}^n \ell'(y_i\tilde y_i) \cdot y_i \cdot \xb_{L-1,i} \vb^\top \bSigma_{L-1}(\xb_i),
\end{align*}
by \ref{item:difference_xbound} and \ref{item:difference_sparsity} in Lemma~\ref{lemma:scaling_perturbation}, with probability at least $1 - \delta / 2 $ we have
\begin{align*}
    \| \nabla_{\Wb_L}L_S (\Wb) - \Gb \|_F &\leq \frac{1}{n} \sum_{i=1}^n [ - \ell'(y_i\tilde y_i) ]\cdot \big\| \big( \bx_{L-1,i} - \bx_{L-1,i}^{(0)} \big) \cdot \vb^\top \bSigma_{L-1}^{(0)}(\bx_i) \big\|_F \\
    &\quad + \frac{1}{n} \sum_{i=1}^n [ - \ell'(y_i\tilde y_i) ]\cdot \big\| \bx_{L-1,i} \cdot \vb^\top \big[ \bSigma_{L-1}(\bx_i) - \bSigma_{L-1}^{(0)}(\bx_i) \big] \big\|_F   \\
    & \leq C_1 \cE_S(\Wb) \cdot \sqrt{m} \cdot (L^2 \tau + L^{2/3} \tau^{1/3}),
\end{align*}
where $C_1$ is an absolute constant. Therefore by the assumption that $\tau \leq C 8^{-L}\cdot L^{-2} \gamma^3 $ for some small enough absolute constant $C$, we have
\begin{align*}
    \| \nabla_{\Wb_L}L_S (\Wb) \|_F \geq 2^{-(L+2)} \cdot \sqrt{m_{L}} \cdot \gamma \cdot \cE_S(\Wb).
\end{align*}
This completes the proof.
\end{proof}

\subsection{Proof of Lemma~\ref{lemma:gradientlowerbound_conjugatekernel}}
\begin{proof}[Proof of Lemma~\ref{lemma:gradientlowerbound_conjugatekernel}]
It follows by Assumption~\ref{assump:conjugatekernelassumption} and Theorem~E.1 in \citet{du2018gradientdeep} that there exists $\balpha\in \RR^{m_{L-1}}$ with $\|\balpha\|_2\leq 1$ such that $y_i\cdot\la \balpha,\xb \ra \geq \gamma$. Therefore, similar to the proof of Lemma~\ref{lemma:gradientlowerbound_randomfeature} and Lemma~\ref{lemma:randinit_gradientlowerbound_empirical}, (essentially, we can treat the $(L-1)$-th layer output as the input of a two-layer network, and apply Lemma~\ref{lemma:gradientlowerbound_randomfeature} and Lemma~\ref{lemma:randinit_gradientlowerbound_empirical} with $L = 1$) we have
\begin{align*}
    \| \nabla_{\Wb_L}L_S (\Wb) \|_F \geq C \gamma \cdot \sqrt{m_{L}} \cdot \cE_S(\Wb),
\end{align*}
where $C$ is an absolute constant. This finishes the proof.
\end{proof}

\section{Proof of Lemmas in Appendix~\ref{section:appendixA}}

\subsection{Proof of Lemma~\ref{lemma:randinit_linearseparable_lastlayer}}

\begin{proof}[Proof of Lemma~\ref{lemma:randinit_linearseparable_lastlayer}]

By Assumption~\ref{assump:nonlinearseparable}, there exists $c(\overline\ub)$ with $\| c(\cdot) \|_\infty \leq 1$ such that
\begin{align*}
    f(\xb) = \int_{\RR^d} c(\overline\ub)\sigma(\overline\ub^\top\xb) p(\overline\ub) \mathrm{d}\overline\ub
\end{align*}
satisfies $y\cdot f(\xb)\geq \gamma$ for all $(\xb,y)\in \supp(\cD)$. 
Let 
$$
\tilde\balpha_1 = (\sqrt{1 / m_1}c(\sqrt{m_1/2} \wb_1^{(0)}),\ldots,\sqrt{1 / m_1}c(\sqrt{m_1/2} \wb_{m_1}^{(0)}))^\top.
$$ 
Since $\| c(\cdot) \|_\infty \leq 1$, we have
$ \| \tilde{\balpha}_1 \|_2^2 = m_1^{-1}\cdot\sum_{j=1}^{m_1} c^2 (\sqrt{m_1/2} \wb_j ) \leq 1$. 
For any $i \in [n]$, we have
\begin{align*}
    \tilde\balpha_1^\top \bx_{1,i}^{(0)} &= \sum_{j=1}^{m_1} \sqrt{\frac{1}{m_1}}c\bigg(\sqrt{\frac{m_1}{2}} \wb_j^{(0)}\bigg) \cdot  \sqrt{\frac{2}{m_1}}\sigma\bigg(\sqrt{\frac{m_1}{2}} \wb_j^{(0)\top} \bx_i\bigg) \\
    &= \frac{\sqrt{2}}{m_1}\sum_{j=1}^{m_1} c\bigg(\sqrt{\frac{m_1}{2}} \wb_j^{(0)}\bigg) \cdot  \sigma\bigg(\sqrt{\frac{m_1}{2}} \wb_j^{(0)\top} \bx_i\bigg).
\end{align*}
Therefore $\EE(\tilde\balpha_1^\top \bx_{1,i}^{(0)} ) = \sqrt{2} f(\bx_i)$. 
Moreover, since $\| c(\cdot) \|_\infty \leq 1$, we have 
$$
\bigg\| c\bigg(\sqrt{\frac{m_1}{2}} \wb_j^{(0)}\bigg) \cdot  \sigma\bigg(\sqrt{\frac{m_1}{2}} \wb_j^{(0)\top} \bx_i\bigg) \bigg\|_{\psi_{2}} \leq C_1
$$
for some absolute constant $C_1$. Therefore by Hoeffding inequality and union bound, with probability at least $1 - \delta/4$ we have
\begin{align*}
    |\tilde\balpha_1^\top \bx_{1,i}^{(0)} - \sqrt{2} f(\bx_i)| \leq C_2 \sqrt{\frac{\log(4en /\delta)}{m_1}} \leq \gamma/2
\end{align*}
for all $i\in [n]$, where $C_2$ is an absolute constant. Set $\balpha_1 = \tilde\balpha_1 / \|\tilde\balpha_1 \|_2$. Then by $\|\tilde\balpha_1 \|_2 \leq 1$, we have
\begin{align*}
    y_i\cdot \balpha_1^\top \bx_{1,i}^{(0)} \geq  \sqrt{2}\cdot \gamma - \gamma / 2 > \gamma/2
\end{align*}
for all $i\in [n]$.  

Now we define $\tilde\balpha_l = \Wb_l^{(0)}\tilde\balpha_{l-1}$, $l=2,\ldots,L$. 
For any $l=2,\ldots,L$, by definition, we have $\| \tilde\balpha_{l} \|_2^2 =  \sum_{j=1}^{m_l} (\wb_{l,j}^{(0)\top} \tilde\balpha_{l-1})^2 $.
Therefore we have $\EE(\| \tilde\balpha_{l} \|_2^2 | \tilde\balpha_{l-1} ) = 2 \| \tilde\balpha_{l-1} \|_2^2$. Since we have $\| (\wb_{l,j}^{(0)\top} \tilde\balpha_{l-1})^2 \|_{\psi_1} = O( \| \tilde\balpha_{l-1} \|_2^2 / m )$, by Bernstein inequality and union bound, with probability at least $1 - \delta/2$, 
\begin{align*}
    \big|\| \tilde\balpha_{l} \|_2^2 - 2 \| \tilde\balpha_{l-1} \|_2^2 \big| \leq  C_3 \| \tilde\balpha_{l-1} \|_2^2\cdot \sqrt{\frac{ \log(4L/\delta) }{ m_l }} \leq 2 \| \tilde\balpha_{l-1} \|_2^2
\end{align*}
for all $l = 2,\ldots, L$, where $C_3$ is an absolute constant. Therefore since $\| \tilde\balpha_{1} \|_2 = 1$, we have $\| \tilde\balpha_{l} \|_2 \leq 2^{l-1}$
for all $l=2,\ldots,L$. Moreover, for any $i\in [n]$ and $l=2,\ldots,L$, by definition, we have $\wb_{l,j}^{(0)\top} \stackrel{d}{=} -\wb_{l,j}^{(0)\top}$, $j\in [m_l]$, and therefore 
\begin{align*}
    \EE[ \la \tilde\balpha_{l}, \bx_{l,i}^{(0)} \ra | \tilde\balpha_{l-1}]
    &= \sum_{j=1}^{m_l} \EE\Big[ \big( \wb_{l,j}^{(0)\top} \tilde\balpha_{l-1} \big) \sigma\big( \wb_{l,j}^{(0)\top} \bx_{l-1,i}^{(0)} \big) \Big| \tilde\balpha_{l-1}\Big]\\
    &= \frac{1}{2}\sum_{j=1}^{m_l} \EE\Big[ \big( \wb_{l,j}^{(0)\top} \tilde\balpha_{l-1} \big) \sigma\big( \wb_{l,j}^{(0)\top} \bx_{l-1,i}^{(0)} \big) - \big( \wb_{l,j}^{(0)\top} \tilde\balpha_{l-1} \big) \sigma\big( - \wb_{l,j}^{(0)\top} \bx_{l-1,i}^{(0)} \big) \Big| \tilde\balpha_{l-1}\Big] \\
    &= \frac{1}{2} \sum_{j=1}^{m_l} \EE\Big[ \big( \wb_{l,j}^{(0)\top} \tilde\balpha_{l-1} \big) \big( \wb_{l,j}^{(0)\top} \bx_{l-1,i}^{(0)} \big) \Big| \tilde\balpha_{l-1}\Big] \\
    &= \la \tilde\balpha_{l-1}, \bx_{l-1,i}^{(0)} \ra.
\end{align*} 
Since 
\begin{align*}
    \big\| \wb_{l,j}^{(0)\top} \tilde\balpha_{l-1} \cdot \sigma(  \wb_{l,j}^{(0)\top} \bx_{l-1,i}^{(0)}  ) \big\|_{\psi_{1}} \leq C_4 \big\| \la  \wb_{l,j}^{(0)} , \tilde\balpha_{l-1} \ra \big\|_{\psi_{2}} \cdot \big\| \la \wb_{l,j}^{(0)} , \bx_{l-1,i}^{(0)} \ra \big\|_{\psi_2}\leq C_5 \| \tilde\balpha_{l-1} \|_2 / m_l,
\end{align*}
where $C_4,C_5$ are absolute constants, by Bernstein's inequality and a union bound, with probability at least $1 - \delta/2$ we have
\begin{align*}
     \big|\la \tilde\balpha_{l}, \bx_{l,i}^{(0)} \ra - \la \tilde\balpha_{l-1}, \bx_{l-1,i}^{(0)} \ra\big| \leq C_6 \| \tilde\balpha_{l-1} \|_2\cdot  \sqrt{\frac{\log( 4nL /\delta ) }{m_l}} \leq \gamma /(4L)
\end{align*}
for all $i \in [n]$ and $l = 2,\ldots, L$, where $C_6$ is an absolute constant. Therefore we have
\begin{align*}
    y_i\cdot \la \tilde\balpha_{l}, \bx_{l,i}^{(0)} \ra \geq y_i\cdot \la \tilde\balpha_{l-1}, \bx_{l-1,i}^{(0)} \ra - \gamma /(4L)\geq \cdots \geq \gamma/2 - \gamma / 4 = \gamma / 4
\end{align*}
for all $i\in [n]$ and $l=2,\ldots,L$. 
Setting $\balpha_{l} = \tilde\balpha_{l} / \| \tilde\balpha_{l} \|_2$, we obtain
\begin{align*}
    y_i\cdot \la \balpha_{l}, \bx_{l,i}^{(0)} \ra \geq 2^{-(l-1)} \cdot \gamma / 4 \geq 2^{-(l+1)} \gamma.
\end{align*}
This completes the proof.
\end{proof}

\subsection{Proof of Lemma~\ref{lemma:randinit_gradientlowerbound_empirical}}

\begin{proof}[Proof of Lemma~\ref{lemma:randinit_gradientlowerbound_empirical}]
By Lemma~\ref{lemma:randinit_linearseparable_lastlayer}, with probability at least $1 - \delta/2$, there exists $\alpha_{L-1}\in S^{m_{L-1} - 1}$ such that $y_i\cdot \la \alpha_{L-1}, \bx_{L-1,i}^{(0)} \ra \geq 2^{-L}\gamma $ for all $ i \in [n]$. Moreover, by direct calculation we have $\EE[\sigma'( \wb_{L,j}^{(0)\top} \bx_{L-1,i}^{(0)} )| \bx_{L-1,i}^{(0)}] = 1/2 $. Therefore, by Hoeffding inequality, with probability at least $1 - \delta / 2$ we have
\begin{align}\label{eq:randinit_gradientlowerbound_empirical_eq1}
    \frac{1}{m_L} \sum_{j=1}^{m_L}\sigma'( \wb_{L,j}^{(0)\top} \bx_{L-1,i}^{(0)} ) 
    \geq \frac{1}{2} -  C_1\sqrt{\frac{\log(n/\delta)}{m_L}} \geq \frac{1}{2\sqrt{2}} > 0
\end{align}
for all $i\in [n]$, where $C_1$ is an absolute constant. 
Hence we have
\begin{align*}
    &\sum_{j=1}^{m_L} \Bigg\| \frac{1}{n} \sum_{i=1}^n [ a(\bx_i,y_i) \cdot y_i \cdot \sigma'( \wb_{L,j}^{(0)\top} \bx_{L-1,i}^{(0)} ) \cdot \bx_{L-1,i}^{(0)} ] \Bigg\|_2^2\\
    &\qquad \geq m_L  \Bigg\| \frac{1}{n} \sum_{i=1}^n \Bigg[ a(\bx_i,y_i) \cdot y_i  \cdot \bx_{L-1,i}^{(0)} \cdot \frac{1}{m_L} \sum_{j=1}^{m_L}\sigma'( \wb_{L,j}^{(0)\top} \bx_{L-1,i}^{(0)} )  \Bigg] \Bigg\|_2^2\\
    &\qquad  \geq m_L  \Bigg[ \frac{1}{n} \sum_{i=1}^n \Bigg\la a(\bx_i,y_i) \cdot y_i  \cdot \bx_{L-1,i}^{(0)} \cdot \frac{1}{m_L} \sum_{j=1}^{m_L}\sigma'( \wb_{L,j}^{(0)\top} \bx_{L-1,i}^{(0)} ), \balpha_{L-1} \Bigg\ra \Bigg]^2 \\
    &\qquad  \geq 4^{-L}\gamma^2 \cdot m_L  \Bigg[ \frac{1}{n} \sum_{i=1}^n  a(\bx_i,y_i) \cdot \frac{1}{m_L} \sum_{j=1}^{m_L}\sigma'( \wb_{L,j}^{(0)\top} \bx_{L-1,i}^{(0)} ) \Bigg]^2 \\
    &\qquad  \geq 4^{-L}/8 \cdot \gamma^2 \cdot m_L \Bigg[ \frac{1}{n} \sum_{i=1}^n  a(\bx_i,y_i) \Bigg]^2,
\end{align*}
where the first inequality follows by Jensen's inequality, the second and third inequality follows by  Lemma~\ref{lemma:randinit_linearseparable_lastlayer}, and the last inequality is by \eqref{eq:randinit_gradientlowerbound_empirical_eq1}. This completes the proof.
\end{proof}

\bibliography{ReLU}
\bibliographystyle{ims}

\end{document}